\documentclass[review,11pt]{ReportTemplate}
\usepackage{bm}
\usepackage{graphicx}
\usepackage{paralist,algorithmic,algorithm}
\usepackage{amsmath,mathrsfs,amssymb}
\usepackage{multirow}
\newtheorem{thm}{Theorem}

\def \x {\mathbf{x}}

\newenvironment{proof}{\textbf{Proof:}\ }{\hspace{\stretch{1}}$\square$\\}

\begin{document}
\begin{frontmatter}
\title{Inductive Logic Boosting}

\author{Wang-Zhou Dai}
\author{Zhi-Hua Zhou\corref{cor1}}
\address{National Key Laboratory for Novel Software Technology\\
Nanjing University, Nanjing 210093, China} \cortext[cor1]{\small Corresponding author.
Email: zhouzh@nju.edu.cn}

\begin{abstract} 
Recent years have seen a surge of interest in Probabilistic Logic
Programming (PLP) and Statistical Relational Learning (SRL) models
that combine logic with probabilities. Structure learning of these
systems is an intersection area of Inductive Logic Programming (ILP)
and statistical learning (SL). However, ILP cannot deal with
probabilities, SL cannot model relational hypothesis. The
biggest challenge of integrating these two machine learning frameworks
is how to estimate the probability of a logic clause only from the
observation of grounded logic atoms. Many current methods models a joint
probability by representing clause as graphical model and literals as
vertices in it. This model is still too complicate and only can be
approximate by pseudo-likelihood. We propose Inductive Logic Boosting
framework to transform the relational dataset into a feature-based
dataset, induces logic rules by boosting Problog Rule Trees and
relaxes the independence constraint of pseudo-likelihood. Experimental
evaluation on benchmark datasets demonstrates that the AUC-PR
and AUC-ROC value of ILP learned rules are higher than current
state-of-the-art SRL methods.
\end{abstract} 

\end{frontmatter}

\section{Introduction}
\label{sec:intro}
In recent years, there has been an increasing interest in
integrating first-order logic with probabilities by defining
confidence of a logic formula with a weight. This interest has resulted in
the fields of Probabilistic Logic Programming (PLP)
\cite{DeRaedt:08:PILP} and Statistical Relational Learning (SRL)
\cite{Getoor:07:introSRL}. PLP focuses on the extension of logic
programming languages with probabilities, as in for instance Problog
\cite{DeRaedt:07:problog}. Conversely, SRL extends first-order logic
on probabilistic graphical models as Markov or Bayesian networks, like
Markov logic Network (MLN) \cite{Richardson:06:mln}. 

Structure learning of PLP and SRL is an important but challenging task
because it defines the complex relationships among entities and have
to be learned in a exponential searching space. The task actually is a
intersection of Inductive Logic Programming (ILP) and Statistical
learning (SL). 

Inductive Logic Programming \cite{Muggleton:91:ilp} is
a combination of inductive learning and logic programming,
which employs techniques from both machine learning and logic
programming. Given a set of training examples of a target
predicate(relation) and background knowledge, ILP finds a hypothesis
which is complete and consistent with evidence and background
knowledge\cite{Lavrac:94:ilpbook}.

ILP can learns relational data which is a collection of logic
relations, while SL deals with datasets of grouped instances
that have distinguishing feature values. The difference between these two
tasks makes the combination of ILP and SL very difficult. However,
structure learning of PLP and SRL is so important that has received
much attention recently, many practical algorithms have been proposed.


Most of these methods assume a joint distribution on a probabilistic
graph model for each logic formula, then learn parameters that maximizing
pseudo-likelihood of the formulas by exploiting the labeled
examples. The learning strategy of these approaches can be
categorized into two kinds. The first category learns logic formulas
and weight separately, for instance as
\cite{Kok:05:MSL,mihalkova:07:BUSL,Kok:09:LHL,Kok:10:LSM}. The other
category turns the problem into a series of relational regression
problems and learns the weights and the clauses simultaneously
\cite{Khot:11:MLN-Boost,Natarajan:12:RDN-Boost}.  

Graphical model based methods introduce pseudo-likelihood by assuming
independence between logical literals, then turns the structure
learning problem into a statistical learning problem and solve it with
many statistical algorithms. However, this kind of transformation does
not exploit the ability of induction of first-order logic. It treats
logic literals independently during learning thus loses accuracy,
it also sacrifices the expressiveness of logic because only uses
statistical model to represent the relational hypothesis.

In this paper, we present Inductive Logic Boosting (ILB), possibly a
new class of approaches to integrate ILP with statistical learning. It
first finds paths for positive examples in a hyper-graph that
constructed from relational database. Then substitutes the paths to
first-order core forms (patterns) to generate binary labeled feature-based
instances. Finally learns an ensemble model through Adaboost the Problog
Rule Trees for target predicate on the generated training data
with specific features. The generated dataset and boosting introduces
characters of statistical learning, Problog Rule Tree preserves
features of ILP. ILB relaxes the assumption of independence and learns an
PLP model in global underlying data distribution. Expreiments shows
ILB produces PLPs more accurate and more comprehensible than current
state-of-the-art approaches.

The remainder of this paper is arranged as follows. We begin by
reviewing related works in \ref{sec:relwork}, then introduce some some
backgrounds in \ref{sec:pre}. We describe the detail of ILB in
\ref{sec:method} and report the experiments in \ref{sec:exp}. Finally,
we conclude with future works in \ref{sec:conc}.

\section{Related Work}
\label{sec:relwork}
The recent years of ILP have been dominated by the development of
methods for learning probabilistic logic representations. A general
framework for Probabilistic ILP (PILP) was introduced
\cite{DeRaedt:08:PILP}. 

Most of current systems integrates ILP and statistical learning by
expressing first-order logic as probabilistic graphic models and then
learn the parameters on the graph models. They search structures
(candidate clauses) first, then learns the parameters (weights) and
modify the structures (clauses) accordingly. This kind of approaches
performs either top-down \cite{Kok:05:MSL} or bottom-up searches
\cite{mihalkova:07:BUSL}. There are also works learns PLP by beam
search or approximate search in the space of probabilistic clauses
\cite{Bellodi:13:SLPLP,Mauro:13:BanditPILP}.

There are also some methods combines ILP with SL by boosting. For
example, Boosting FFOIL \cite{Quinlan:96:boostfoil} directly adopts
the boosting framework with a classical ILP system, FFOIL, as weak
learners, it proves that boosting is beneficial for first-order
induction. More recently, RDN-Boost \cite{Natarajan:12:RDN-Boost} and
MLN-Boost \cite{Khot:11:MLN-Boost} turns the problem into relational
regression problems and learns both structures and weights of
graphical model simultaneously.

Different with previous methods, Inductive Logic Boosting transforms
relational dataset to a feature-based dataset, then learns Problog Rule
Trees by discriminative learning to induce both first-order logic
rules and their weights, finally use Adaboost to get an accurate
hypothesis definition of target predicates.

\section{Preliminary}
\label{sec:pre}

In first-order logic, \textit{formulas} are constructed by four types of
symbols: \emph{constants}, \emph{variables}, \emph{functors} and
\emph{predicates}. In this paper we follow the Prolog terminology that
using words begin with a lowercase letter to represent \emph{constants} and
\emph{predicates}, words begin with a uppercase letter to represent
\emph{variables}. A \textit{term} is a variable, a constant, or a
functor applied to terms. An \textit{atom} is of the form
$p(t_1,\dots,t_n)$ where $p$ is a predicate of arity $n$ and the $t_i$
are terms. A \textit{formula} is built out of atoms using using quantifiers
$\forall,\exists$ and usual logical connectives
$\neg,\wedge,\vee,\rightarrow$ and $\leftrightarrow$. A \textit{rule}
(also called a \textit{normal clause}) is a universally quantified
formula of the form $\mathsf{h\texttt{:-}t_1,\dots,t_n}$, where
atom $\mathsf{h}$ is called the \textit{head} of the rule and literals
$\mathsf{t_1,\dots,t_n}$ the \textit{body} where $t_i$ are logical
atoms. The formula means the conjunction $\mathsf{t_1\wedge\dots\wedge
  t_n}$ will deduce $\mathsf{h}$. A \textit{logic program (LP)} is a
set of FOL rules. A \textit{fact} is a \textit{rule} with an empty
\textit{body} and is written more compactly as $\mathsf{h}$, means
$\mathsf{h}$ is always true in the program. 

The task of inductive logic boosting is similar with ILP, which can be 
formally put as this: Given (i) a set of training examples $\mathcal{E}$,
including \textit{true} groundings $\varepsilon^+$ and \textit{false}
groundings $\varepsilon^-$ of a target predicate(relation) $p$;
(ii) a description language $\mathcal{L}$, specifying syntactic
restrictions on the definition of predicate on the definition of
predicate $p$; (iii) background knowledge $\mathcal{B}$, defining
other predicates $q_i$ that may be used in definition of $p$. Find a
hypothesis $\mathcal{H}$ as the definition of $p$, which can predict
the confidence of each grounding of $p$.

Problog is one of Probabilistic Logic Programming languages. It
integrates logic program with probability by adding a probability
$\mathsf{p}$ to each ground facts $\mathsf{f}$, written
$\mathsf{p}::\mathsf{f}$. It also allows \textit{intentional}
probabilistic statements of the form
$\mathsf{p\texttt{::}}p(A_1,A_2,\dots,A_n)\texttt{:-}\mathsf{body}$,
where $p(\cdot)$ is a probabilistic atom as head, $\mathsf{body}$
is a conjunction of calls to non-probabilistic facts. Like Prolog, the
rules are range-restricted: all variables in the head of a rule should
also appear in a positive literal in the body of the rule.

A Problog program specifies a probability distribution over Hebrand
interpretation, or \textit{Possible World}. The ground probabilistic
fact $\mathsf{p\texttt{::}f}$ gives an \textit{atomic choice}, it
means the program can choose to include $\mathsf{f}$ as a fact with
probability $\mathsf{p}$ or reject it with probability
$1-\mathsf{p}$. A \textit{total choice} is obtained by making an
atomic choice for each ground probabilistic fact. The probability
distribution over the total choices is defined to be the product of
the probabilities of the atomic choices that it is composed of as
independent events, i.e. there are two probabilistic facts $0.3::a$
and $0.8::b$, then the total choices are $\{a,b\},\{a\},\{b\}$ and
$\{\}$ with probabilities $0.24,0.06,0.56$ and $0.14$.

\section{Proposed Approach}
\label{sec:method}
We now present the Inductive Logic Boosting framework for integrating
logic induction and statistical learning. ILB represents Problog
rules as decision trees, which we called \emph{Problog Rule Tree}.
It learns rules and weights simultaneously from an alternated
dataset, which is a collection of a binary labeled instances generated
from original relational data. Finally boosts on these rule sets
adaptively to provide an accurate hypothesis.

\subsection{Problog Rule Tree}
\label{sec:PRT}

A \emph{Problog Rule Tree} $T=(h,N,C)$ is a kind of Relational
Probability Tree (RPT)
\cite{Neville:03:RelProbTree}. $h=p(A_1,\dots,A_n)$ is the \emph{head}
of $T$, represents the target predicate $p(A_1,\dots,A_n)$ to be
earned, $\{A_i\}$ are the arguments of $p$; $N={n_1,\dots,n_l}$ is
the node set of $T$. Each node $n$ is either a decision node or end
node. Each decision node $n^d=\{t_1,\dots,t_k\}$ is a conjunction of
logic atoms, it has a \emph{true child} and a \emph{false child} to
determine whether an instance satisfies $n^d$, end node $n^e$ records
the proportion of positive instances which satisfy all its true
ancestors. If $n'$ is the true child of $n$, then $n$ is called
\emph{true parent} of $n'$; $C=\{t_1,\dots,t_m\}$ is the tree root,
together with $h$ they formulate a short logic rule $h\mathtt{:-}C$
which we call \emph{core form} of $T$. $C$ is the first decision node
of a Problog tree and only has a true child. Therefore, one Problog
Rule Tree is learned to expand only one core form. In order to make
the learned rule legal in Problog, we constrain that all variables
$A_i$ appear in $h$ must also appear in $C$. For example, figure
\ref{fig:PRT} denoted a tree to expand \emph{core form}
\emph{sametitle}\texttt{:-}\emph{hasword(X,Z1),hasword(Y,Z1)}.

\begin{figure}[t]
  \centering
  \includegraphics[width=0.5\textwidth]{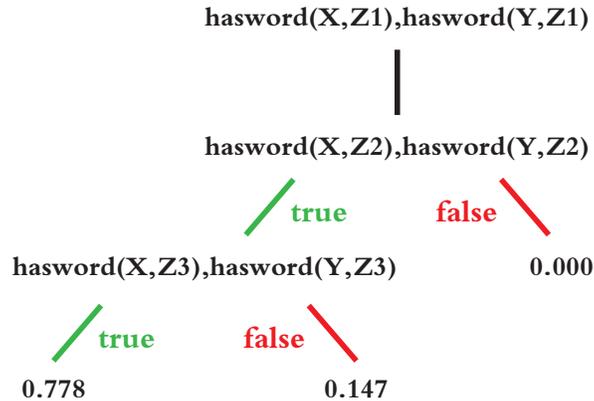}  
  \caption{An Example of learned Problog Rule Tree for predicate
    $h$=\textbf{\emph{sametitle(X,Y)}}: the root conjunction is the
    \emph{core form} $c$ that to be expanded, which means $X$ and $Y$
    have a common word $Z_1$. Other clauses are decision nodes. The
    end nodes 0.000, 0.147, 0.778 correspond to the proportion of true
    \emph{sametitle(X,Y)} that have only one, two or three common
    words, the proportions also are the probability of corresponding
    Problog rules.}
  \label{fig:PRT}
\end{figure}


From each end node in $T$, we can restore a Problog statement by
backtracking its route to the root. For example in figure
\ref{fig:PRT}, we can get a statement from the 0.147 end node:

\begin{tabular*}{0.4\textwidth}{l l l}
\multicolumn{2}{l}{0.147\texttt{::}\emph{sametitle(X,Y)} \texttt{:-}} \\
 & \hspace{0.15\textwidth}\emph{(hasword(X,Z1), hasword(Y,Z1))},\\
 & \hspace{0.15\textwidth}\emph{(hasword(X,Z2), hasword(Y,Z2))},\\
 & \hspace{0.15\textwidth}\emph{\textbackslash+((hasword(X,Z3),
   hasword(X,Z3)))},\\ 
 & \hspace{0.15\textwidth}\emph{unique([Z1,Z2,Z3])}.
\end{tabular*}

which means if $X$ and $Y$ only have 2 common words, the
probability of \emph{sametitle(X,Y)} is $0.147$ (This rule is only for
demonstration, the actual problog rule used in ILB will ignore negations of
conjunctions). 

The weighted proportion of positive instances in an end node is used
as confidence of the statement. This is because for an individual
tree, Prolog rule extracted from each end node covers different part
of data (which is the nature of decision tree). Without
overlap in the underlying distribution, the proportion of
positive instances is the maximum likelihood estimation for the 
confidence of those rules.

Remind that each tree only deals with the formulas learned from one
core form, although the expanded Problog rules from same core form has
no intersection in their coverage, the trees who expand different core
forms may cover same examples. For instance in entity resolution task,
an example \emph{sameauthor(person1,person2)} may be covered by rule
``$0.2$\texttt{::}\emph{sameauthor(X,Y)}\texttt{:-} \emph{hasword(X,Word),
  hasword(Y,Word)}.'' and rule ``$0.3$\texttt{::}\emph{sameauthor(X,Y)}
\texttt{:-} \emph{author(X,Title1), sametitle(Title1,Title2),
  author(Y,Title2)}.'' at the same time. We followed the Problog
settings to use \textbf{noisy or} operation to estimate the joint
probability:

\begin{equation}
  \label{eq:noisyor}
  P(X|R_1,\dots,R_k,\neg R_{k+1},\dots,\neg R_n)=1-\prod_{i=1}^k(1-p_i)
\end{equation}
where $p_i$ is the probability of rule $R_i$ to be true. Noisy or is a
probabilistic generalization of the logical or. In the previous example
$P(sameauthor(person1,person2)=true)=1-(1-0.2)(1-0.3)=0.44$.



\subsection{Structure and Parameter Learning}
\label{sec:learn}

ILB can Learn Problog rule trees through many simple decision tree
learners, for instance as C4.5 \cite{Quinlan:93:C45} or CART
\cite{Breiman:84:CART}. In order to make these learners feasible, ILB
will turn relational data into a feature-based discriminative dataset
at first.

When learning a target predicate $p$, we define an instance $x$ of ILB
is a pair $(\mathtt{P},y)$ consists of label $y\in\{-1,1\}$ and a
conjunction of some grounded logical atoms
$\mathtt{P}=t_1\wedge\dots\wedge t_n$. Notice the set of
\textbf{generated instances} $E$ is different from the set of
pre-labeled training examples $\mathcal{E}$ defined in section
\ref{sec:pre}. The outline of instance generation procedure is
presented in algorithm \ref{alg:instgen}.


The procedure of feature-based data extraction is based on a hypergraph
generated from the original data. Relational database can be viewed as
a hypergraph with constants as nodes, and true ground atoms as
hyperedges. Therefore, first-order rule which defines a goal concept
can be viewed as a template subgraph consists of numerous variabilized
hyperedges. 

Logical induction is actually to find such a template subgraph that
could match (or cover) as much as positive examples and as few as the
negative examples. However, the hypothesis space of subgraph structure
is so huge and complex that makes ordinary searching algorithms
intractable. 




ILB constrains the searching space by relational path finding
\cite{Richards:92:pathfind}. It is based on the
assumption that there usually exists a fixed-length path of relations
linking the set of terms that satisfying the goal concept. This
approach achieves many good results in inductive logic programming
area and inspired lots of SRL structure learning algorithms
\cite{mihalkova:07:BUSL,Kok:09:LHL}. ILB uses paths of positive
examples as the \emph{core forms} from which it expands the searching
space to find the optimal solution.



\begin{figure}[t]
  \centering
  \includegraphics[width=0.6\textwidth]{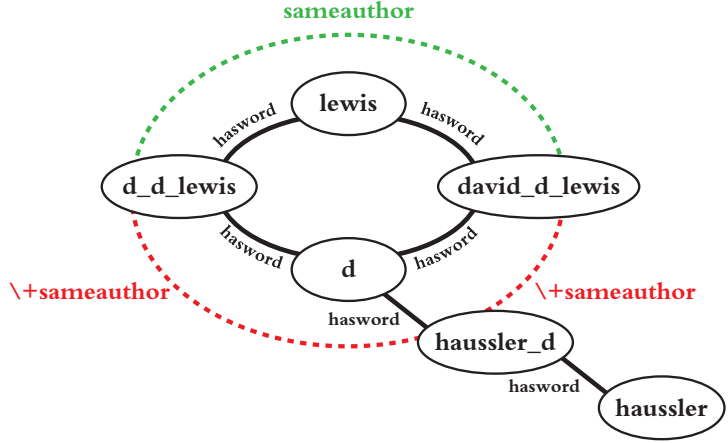}  
  \caption{Example in Cora dataset of learning \emph{sameauthor(X,Y)}:
      There are 3 people in this domain, their name have a common word
      ``d'', but \emph{haussler\_d} is different with the
      other nodes. ``\textbackslash+'' means negation in Prolog clause.}
  \label{fig:instgen}
\end{figure}

At the beginning of learning Problog Rule Trees, we construct a
hypergraph $\mathcal{G}=(E(\mathcal{G}),V(\mathcal{G}))$ from the
relational database $D$, where $E(\mathcal{G})$ is the set of
hyperedges $\{t_i(v_1,\dots,v_k), k=arity(t_i)\}$ (denoted by $t$
because they are also \emph{ground terms} in $D$), $V(\mathcal{G})$ is
the set of all vertices in $\mathcal{G}$. This step is described as
initialization step in algorithm \ref{alg:instgen}.

Second, ILB uses a depth-first $FindPath$ procedure to search paths
that contains the nodes appear in a positive examples of
$e^+=\varepsilon^+$. In the example showed in figure \ref{fig:instgen},
given a positive example
\emph{sameauthor(d\_d\_lewis,david\_d\_lewis)}, $PathFinding$ will
finds 2 grounded paths:
$\mathtt{P}_1$=\emph{\{hasword(d\_d\_lewis,lewis),
  hasword(david\_d\_lewis, lewis)\}} and
$\mathtt{P}_2$=\emph{\{hasword(d\_d\_lewis, d),
  hasword(david\_d\_lewis, d)\}}.

Third, variabilize the retrieved paths to get a
\textbf{core form} $h\mathtt{:-} C$, formally put as
$\theta e$ \texttt{:-} $\theta\mathtt{P}$, where $\theta=[a_i/A_i]$ is
a substitution replaces each unique constant appear in $\mathtt{P}$ with a
variable. In the above example we'll get only one core form
\emph{sameauthor(X,Y)} \texttt{:-} \emph{\{hasword(X,Z),
  hasword(Y,Z)\}}. This step is described as $Substitute$
procedure in algorithm \ref{alg:instgen}.

Fourth, to capture the underlying global distribution in domain, ILB
uses the \emph{core forms} to match the whole hypergraph to find all
grounded paths that satisfy it, this step can be efficiently done by
querying the first-order clause of the core form in the relational
database, shown as $QueryProlog$ procedure in algorithm
\ref{alg:instgen}. In the example, we will get four grounded paths:
the former two paths $\mathtt{P}_1$, $\mathtt{P}_2$ and two new paths
$\mathtt{P}_3$ = \emph{\{hasword(haussler\_d,d),
  hasword(d\_d\_lewis,d)\}} and $\mathtt{P}_4$ =
\emph{\{hasword(haussler\_d,d), hasword(david\_d\_lewis,d)\}} (In this
example there will be another four symmetric paths, here we omit them
as trivial situations). 

Finally, we construct an instance $x=(\mathtt{P},y)$ for training data
as follows: for each retrieved path $\mathtt{P}$ check its deduced
ground atom deduced from the core form, which is described
as $PrologDeduce$ in algorithm \ref{alg:instgen}. If the grounding
belongs to $\varepsilon^+$ then assign $y=1$, if it belongs to
$\varepsilon^-$ or does \emph{not appears} in $\varepsilon$ then
$y=-1$ (This follows the closed world assumption). Hence the algorithm
can generate 4 instances in our example, $\mathtt{P}_1$ and
$\mathtt{P}_2$ that deduce
\emph{sameauthor(d\_d\_lewis,david\_d\_lewis)} will be marked as
positive instances; $\mathtt{P}_3$ and $\mathtt{P}_4$ deduce
\emph{sameauthor(haussler\_d,david\_d\_lewis)} and
\emph{sameauthor(haussler\_d,d\_d\_lewis)}, which are not in
$\varepsilon^+$, so they will be marked as negative instances.

\begin{algorithm}[h!]
   \caption{\emph{GenerateInstance($\mathcal{B},\varepsilon^+$)}}
   \label{alg:instgen}
\begin{algorithmic}
   \STATE {\bfseries Input:} Background knowledge (Relational
   Database) $\mathcal{B}$, positive examples $\varepsilon^+$
   \STATE {\bfseries Output:} Training data $D=\{(\mathtt{P}_i,y_i)\}$
   \STATE Initialize hypergraph $\mathcal{G}$ with $\mathcal{B}$.
   \STATE $D=\phi$.
   \STATE $coreForms=\phi$
   \FOR {each example $e\in\varepsilon^+$}
   \STATE $paths=FindPath(G,e)$
   \FOR {each path $\mathtt{P}\in paths$}
   \STATE core form $c=Substitute(\mathtt{P})$
   \IF {$C\notin coreForms$}
   \STATE add $C$ to $coreForms$
   \ENDIF
   \ENDFOR
   \ENDFOR
   \FOR {each core form $C\in coreForms$}
   \STATE $allPaths=QueryProlog(\mathcal{B},c)$
   \FOR {each queried output $\mathtt{P}_i\in allPaths$}
   \STATE $head=PrologDeduce(P_i,c)$
   \IF {$head\in \varepsilon^+$}
   \STATE $y_i=1$
   \ELSE 
   \STATE $y_i=-1$
   \ENDIF
   \STATE
   $(\Phi_b(\mathtt{P}_i),\Phi_p(\mathtt{P}_i))=computeFeature(\mathtt{P}_i)$ 
   \STATE
   $x_i=generateInstance(c,\mathtt{P}_i,y_i,\Phi_b(\mathtt{P}_i),\Phi_p(\mathtt{P}_i))$
   \STATE add $x_i$ to $D$
   \ENDFOR
   \ENDFOR
   \STATE {\bfseries Return:} $D$
\end{algorithmic}
\end{algorithm}

After we get the labeled instances, a more important task for rule
induction by statistical learning is to calculate features. ILB uses
two kinds of feature for each generated instance $x=\{\mathtt{P},y\}$:

\begin{itemize}
\item \textbf{Branch Feature}: A \emph{branch feature}
  $\Phi_b(\mathtt{P},v_i)$ of node $v_i$ in path $\mathtt{P}$ is a
  substituted path $\theta\mathtt{P}'$, in which $\mathtt{P}'$ starts
  from $v_i$ and only share one node with $\mathtt{P}$. Thus it looks
  like a tree branch spread out from $\mathtt{P}$. $\theta$ is the
  same substitution that maps $\mathtt{P}$ to its core form
  $C(\mathtt{P})$. In the example of figure \ref{fig:instgen},
  instance $\mathtt{P}_4$ has a branch feature
  \emph{hasword(X,haussler)}.

\item \textbf{Path Feature}: some vertices in path $\mathtt{P}$ might be also
  connected by paths other than $\mathtt{P}$. These paths have more than one
  shared vertices with $\mathtt{P}$, thus each one of them can form
  loops with some edges in $\mathtt{P}$. We substitute them with
  $\theta$ and variabilize all other unique constants to construct
  \emph{path features}. Without loss of generality, we
  define \emph{path feature} $\mathtt{P}'=\Phi_p(\mathtt{P})$ are the
  paths that their start node and end node are the only $2$ shared nodes with
  $\mathtt{P}$. If $\mathtt{P}'$ have $n>2$ shared nodes with
  $\mathtt{P}$, then it can be split into $n-1$ parts that each part
  only shares $2$ nodes with $\mathtt{P}$. In the previous example, instance
  $\mathtt{P}_1$ has a path feature
  \emph{\{hasword(X,Z1),hasword(Y,Z1)\}} where \emph{Z1=d} is
  different from \emph{Z=lewis} formulates another path from $X$ to
  $Y$.
\end{itemize}

The \emph{branch feature} $\Phi_b(\mathtt{P},v_i)$ represents
individual property of each node $v_i$ in path $\mathtt{P}$, the
\emph{path feature} $\Phi_p(\mathtt{P})$ captures auxiliary relations
of the nodes within $\mathtt{P}$.

With these two kinds of features, Inductive Logic Boosting
can expand the core form to estimate the goal concept by greedily
adding the best feature step by step. 

\begin{thm}
\emph{(Completeness of Features)}
\label{th:featcomp}
If the optimal body of the goal predicate $p(A_1,\dots,A_n)$ can
be represented by a variabilized hypergraph $\mathcal{G}_p$, then
starting a expansion from any core form $C$ that contains all
arguments $A_i$ of $p$, we can find all other hyper edges
belongs to the optimal $\mathcal{G}_p$ only by exploring and
variabilizing those paths in $\Phi_b(\mathtt{P})$ and
$\Phi_p(\mathtt{P})$, where $\mathtt{P}$ is a grounded path of any
positive instances $x=(\mathtt{P},1)$ of predicate $p$. 
\end{thm}

\begin{proof}
Notice that a positive instance $\mathtt{P}$ is generated from a positive
example $p(a_1,\dots,a_n)$ which unifies with the goal predicate
$p(A_1,\dots,A_n)$, there exists a core form $C\in\mathcal{G}_p$ whose
body can be grounded to $\mathtt{P}$. Moreover, since $\mathtt{P}$ is
positive, there also exists a grounded graph $g_p$ that
contains $\mathtt{P}$ and unifies with the optimal first-order graph
$\mathcal{G}_p$. Because the number of vertices in $\mathcal{G}_p$ is
finite, the optimal substitution $\theta^*$ from node to variables
can be easily found. Therefore, the goal of this proof has been
reduced to prove the grounded graph $g_p$ is reachable when expanding
$x$ with only $\Phi_b(\mathtt{P})$ and $\Phi_p(\mathtt{P})$.

Because $\mathcal{G}_p$ is connected, so after substitution $\theta^*$
the grounded graph $g_p$ remains connected. Hence for each
hyper edge $t'\in g_p \wedge t'\notin \mathtt{P}$, there must exists
at least one path $\mathtt{P}'$ that starts from $t'$ and ends with a
node $v\in \mathtt{P}$. So for all hyper edges $t'\in g_p$,
there exists $\mathtt{P}'\in\Phi_b(\mathtt{P})$ that contains
$t'$. For the hyper edges $t'$ who can connect $\mathtt{P}$ through
more than one paths in $\Phi_c(x)$, e.g. $\mathtt{P}'_1$ and
$\mathtt{P}'_2$, we can construct a \emph{path feature} by connecting
$\mathtt{P}'_1$, $\mathtt{P}'_2$.
\end{proof}

We can see that with only $\Phi_b$ is already complete for
searching the optimal rule of goal concept. However, the number of $\Phi_b$
increases exponentially with the length of path, which means the
information conveyed by \emph{branch feature} decreases exponentially by
its length. Conversely, \emph{path features} is more informative. They
provide auxiliary relation information. Like the example in figure
\ref{fig:instgen}, to learn a predicate \emph{sameauthor(X,
  Y)} start from core form \emph{\{hasword(X, Z), hasword(Y, Z)\}},
apparently a path feature \emph{\{hasword(X, Z1), hasword(Y,Z1)\}} as
more important that branch features like \emph{\{hasword(X, Z2)\}}. 

The $computeFeature$ procedure in algorithm \ref{alg:instgen}
represents this step. Finally with all paths and features we computed,
the training instances are generated and added to training data.
For the example in figure \ref{fig:instgen}, a part of finally
generated data for learning predicate \emph{sameauthor(X,Y)} is
displayed in table \ref{tab:instgen}.

\begin{table}[h]
  \centering
\begin{tabular}{ | c | c | l | }
\cline{1-3}
Core Form & \multicolumn{2}{|c|}{\emph{sameauthor(X, Y)}\texttt{:-}\emph{hasword(X, Z), hasword(Y, Z).}} \\ 
\cline{1-3}
\multirow{6}{*}{$x_1$} & $y$ & 1\\ \cline{2-3}
 & $\mathtt{P}$ & \emph{\{hasword(d\_d\_lewis,lewis),
   hasword(david\_d\_lewis,lewis)\}}.\\ \cline{2-3}
 & \multirow{2}{*}{$\Phi_b$} & \emph{\{hasword(X,d)\}}.\\
 & & \emph{\{hasword(Y,d)\}}.\\ \cline{2-3}
 & $\Phi_p$ & \emph{\{hasword(X,Z1),hasword(Y,Z1)\}}.\\
\cline{1-3}
\multirow{6}{*}{$x_2$} & $y$ & -1\\ \cline{2-3}
 & $\mathtt{P}$ & \emph{\{hasword(haussler\_d,d), hasword(david\_d\_lewis,d)\}}.\\ \cline{2-3}
 & \multirow{3}{*}{$\Phi_b$} & \emph{\{hasword(X,haussler)\}}.\\
 & & \emph{\{hasword(Y,lewis)\}}.\\
 & & \emph{\{hasword(Y,david)\}}.\\ \cline{2-3}
 & $\Phi_p$ & -- \\
\cline{1-3}
\end{tabular}
  \caption{Example of generated instance in Cora dataset}
  \label{tab:instgen}
\end{table}

From the generated labeled data, we can learn a Problog Rule Tree for
each \emph{core form} by any decision tree learner. 


\subsection{Boosting}
\label{sec:adaboost}

Inductive Logic Boosting learns a discriminative task rather than
doing logic induction by calculating coverage. This feature makes
ILB more fits to boosting framework.


ILB use the confidence-based Adaboost \cite{Schapire:99:ImpBoostConf}
in boosting stage. However, the weak hypothesis $h_t$ of each round
boosting is a set of Problog rules, which are expanded from different
core forms. To combine them, ILB use the \emph{noisy or} feature
we have mentioned before. During each evaluation of current hypothesis
$h_t$, for all instances $x\in Instance_e=\{x|PrologDeduce(x,C)=e\}$
that can deduce same example $e\in\varepsilon$ with its core form, ILB
assigns the probability $P(x)=NoisyOr(h_t(x_1),\dots,h_t(x_m)),\forall x_i\in
Instance_e$. Notice that there might be some instances
$x_k=(\mathtt{P}_k,y_k)$ that $h_t$ not covered, follows the
\emph{closed world assumption}, we define $h_t(k)=0$. 

\section{Experiments}
\label{sec:exp}

\subsection{Datasets}
\label{sec:dataset}
We carried out experiments on two real world datasets to investigate
whether ILB performs better than previous state-of-the-arts
SRL approaches. The task is to learn target predicate definitions with
evidence predicates. Both datasets are publicly available at
http://alchemy.cs.washington.edu. 

\begin{table}[t]
  \label{tab:data}
  \centering
  \begin{tabular}{ | c | c | c | c | c | c | }
  \hline
  Dataset & Types & Constants & Predicates & True Atoms & Total Atoms \\
  \hline
  Cora & 5 & 3,079 & 10 & 42,558 & 687,422 \\
  \hline
  UW-CSE & 9 & 929 & 12 & 2112 & 260,254 \\
  \hline
  \end{tabular}
  \caption{Detail of datasets}
\end{table}

\textbf{Cora}. This dataset is a collection of citations to computer
science papers, created by Andrew McCallum, and later processed by
Singla and Domingos \cite{Poon:07:MLNCora} into 5 folds for the task
of duplicating the citations. Evidence predicates are other relations like
\emph{author(Bib,Author)},\emph{title(Bib,Title)},
\emph{venue(Bib,Venue)}, and so on. Relations in this domain is simple
and clear.

\textbf{UW-CSE}. This dataset was prepared by Richardson and Domingos
\cite{Richardson:06:mln}, describes relationships in an academic
department. The dataset is divided into 5 independent areas/folds (AI,
graphics, etc.). The evidence predicates describe students, faculty,
and their relationships (e.g, \emph{Professor(person)},
\emph{TaughtBy(course, person, quarter)}, etc.).  Target predicate is
\emph{advisedBy(Person, Person)}. We omitted 9 equality predicates
follows \cite{Kok:10:LSM}. Relational structure of this domain is more
complex since predicate and arguments are more complicate than Cora
domain.

The detail of each dataset is showed in \ref{tab:data}. Cora has more
constants but has a simpler and clearer relation structure, UW-CSE is a
more complex relational model (hyper graph).

\subsection{Compared Methods}
\label{sec:compare}

We compared Inductive Logic Boosting to following
state-of-the-art systems:

\textbf{RDN-Boost} \cite{Natarajan:12:RDN-Boost}. This algorithm represents a
Relational Dependency Network (RDN) model as regression trees and learns by
boosting. It turns the structure learning problem into a
series of relational function-approximation problems and solves by
gradient-boosting, which easily induces highly complex features over
several iterations and in turn estimate quickly a very expressive
model. This work outperforms numbers of state-of-the-art SRL
structure learning algorithms. Base on this system there is also a
modified version for learning Markov Logic Network
\cite{Khot:11:MLN-Boost}. This approach is denoted as RDN-B.

\textbf{Learning MLN structure by Structure Motif} \cite{Kok:10:LSM}. Key
insight of this approach is that relational data usually contains
recurring patterns, which is called structural motifs. By constraining the
search for clauses to occur within motifs, it can greatly speed up the
search and thereby reduce the cost of finding long clauses(i.e.,
formulas with more than 4 or 5 literals). We use LSM to denote it.

In order to make the comparison as fair as possible, we used the
following protocol. For RDN Boost, we use the default parameter
setting that constrain maximum tree hierarchy to be 4, each node
contains 2 literals at most and boosting for 20 turns. For LSM, we
employ the parameter suggested in \cite{Kok:10:LSM}. ILB searches uses same
hierarchy and node literal length (in ILB we constrain the path length in
feature $\Phi_b$ and $\Phi_p$) limit settings as RDN Boost since they
both learn clauses based on boosting weak learners in tree
structure. Besides, we constrained the max \emph{core form} length for
Cora and UW-CSE to be 4 and 2 accordingly.

Notice that both RDN and LSM have to input background knowledge of
constant types and predicate forms (e.g. \emph{author(Bib,Author)}
indicates predicate \emph{author} only can take \emph{Bib} and
\emph{Author} as arguments in exact those position). Further more, for
RDN approaches we enumerated all possible predicate ``modes''
(e.g. \emph{samebib(`Bib,+Bib)} indicates the first \emph{Bib} can be
not in the head of learned clause) to ensure the completeness while
learning clauses. ILB does not need to use predefined predicate
formulation or variable types. However, the instance generation
procedure always produces huge number of instances (especially
negative instances). Therefore, we randomly sample a part of them
during instance generation. In Cora task, we randomly generate 1000
instances for each \emph{core form} in 1 fold, and 300 in UW task. 

\subsection{Results}
\label{sec:result}

\begin{table*}[t]
\centering
\begin{tabular}{ | l | c | c | c | c |}
\hline
& \multicolumn{2}{|c|}{SameAuthor} & \multicolumn{2}{|c|}{SameBib} \\
\hline
System & AUC-PR & AUC-ROC & AUC-PR & AUC-ROC \\
\hline
ILB & $\mathbf{0.9517\pm0.05}$ & $\mathbf{0.9835\pm0.02}$ &
$\mathbf{0.9576\pm0.04}$ & $\mathbf{0.9967\pm0.00}$ \\
\hline
RDN-B & $0.8094\pm0.14$ & $0.8877\pm0.13$ & $0.9046\pm0.03$ &
$0.9475\pm0.02$ \\
\hline
LSM & -- & -- & -- &
-- \\
\hline
& \multicolumn{2}{|c|}{SameTitle} & \multicolumn{2}{|c|}{SameVenue} \\
\hline
System & AUC-PR & AUC-ROC & AUC-PR & AUC-ROC \\
\hline
ILB & $\mathbf{0.7668\pm0.09}$ & $\mathbf{0.9784\pm0.02}$ &
$\mathbf{0.6696\pm0.11}$ & $\mathbf{0.9606\pm0.01}$ \\
\hline
RDN-B & $0.1424\pm0.05$ & $0.7790\pm0.06$ & $0.0855\pm0.03$ &
$0.5698\pm0.03$ \\
\hline
LSM & -- & -- & -- &
-- \\
\hline
\end{tabular}
\caption{Results on Cora dataset}
\label{tab:cora}
\end{table*}

\begin{table}[t]
\centering
\begin{tabular}{ | l | c | c |}
\hline
& \multicolumn{2}{|c|}{advisedBy} \\
\hline
System & AUC-PR & AUC-ROC \\
\hline
ILB & $0.6192\pm0.16$ & $0.8932\pm0.08$ \\
\hline
RDN-B & $0.6140\pm0.21$ & $0.9549\pm0.03$ \\
\hline
LSM & $0.22\pm0.02$ & $0.62\pm0.06$ \\
\hline
\end{tabular}
\caption{Results on UW-CSE dataset}
\label{tab:uw}
\end{table}

We evaluate these approaches not only by the labeled positive and
negative examples. Actually, there always be many falsities
for a target concept that not covered by the labeled
examples. Thus we evaluate the predicate concept induction task in global
distribution by treating all relations that not appears in
labeled data as negative examples.

Consider the learned definition of target predicate as a binary
classifier, we choose AUC-ROC to compare these approaches.
However, a key property of most relational data sets is the number of
negatives can be order of magnitude more than the number of
positives. To ignore the impact of the overwhelming true negatives, we
also use area under precision-recall curves (AUC-PR) to evaluate the
performance.

The evaluation result of Cora dataset is showed in table
\ref{tab:cora}. In this task LSM only produces trivial
unit clauses, so ILB is only compared with RDN-Boost. We can see that
both AUC-PR and AUC-ROC value of ILB are significantly better than
RDN-Boost. A major reason is that the RDN-B learns a graphical model
by maximizing pseudo-likelihood which assumes independence between
all random variables (logical literal) in logic formulas, while ILB
directly uses the empirical probability to estimate possibility of a
rule being satisfied, which results in better estimation of the goal
concepts.

Result on UW-CSE is presented in table \ref{tab:uw}. In
facts, due to the high complexity in hypergraph generated in by
UW-CSE, path finding in the relational graph will get so many paths,
which results in more than 100 core forms and millions of
feature-based instances and path features, which hardly can be handled
by ILB. Thus, we only samples 5\% of instances and 10\% of features to
do Problog Tree induction. With a highly incomplete training data, the
result is still comparable with RDN-Boost. LSM performs worst because
we did not run another round weight learning procedure as \cite{Kok:10:LSM}
suggests. But even if we do, the performance of LSM should also be
worse than RDN-Boost \cite{Khot:11:MLN-Boost}.

We can also observe that the results in table \ref{tab:cora} for
RDN-Boost and LSM are worse than those reported by
\cite{Natarajan:12:RDN-Boost}, \cite{Khot:11:MLN-Boost} and
\cite{Kok:10:LSM}. They learn a predicate with all
other predicates as evidences. For instance, when learning
\emph{SameBib} in Cora, \emph{SameAuthor} and \emph{SameTitle} can be
used in body of learned hypothesis. Our task on predicate concept
induction is much more challenging since we do not use any predicates
that can be superseded by other predicates in the domain. This ability
ensures us to discover novel knowledge from the most basic concepts in
a domain.

\section{Conclusions}
\label{sec:conc}
We presented Inductive Logic Boosting, which learns
weighted logical rules in a statistical learning framework.
It uses path-finding in relational domain to generate a discriminative
labeled dataset and calculates two kinds of features. Then performs
decision tree boosting, a simple yet effective statistical learning
algorithm, to learn Problog rules. Our empirical comparisons with two
state-of-the-art systems on real datasets demonstrate the
effectiveness of ILB.

As future work, we want to generalize the ILB framework to accomplish
more logic induction tasks like predicate invention and learn
recursive rules, to provide a different angle of view for inductive logic
programming.
 
\bibliographystyle{alpha}
\bibliography{ml,ilp}
\end{document}